\documentclass{article}
\usepackage[preprint]{nips}
\usepackage{graphics}

\usepackage{natbib}
\usepackage{color} 
\usepackage{acronym} 
\usepackage{hyperref}
\usepackage{nicefrac} 
\usepackage{epsfig} 
\usepackage{amsmath} 
\usepackage{amsthm}
\usepackage{amssymb}  
\usepackage{algorithm, algorithmicx}
\usepackage[noend]{algpseudocode}
\usepackage{booktabs} 

\newtheorem{dfr}{Definition}
\newtheorem{theoremm}{Theorem}

\newtheorem{preposition}{Preposition}

\newcommand{\reals}{\mathbb{R}}

\title{Reward Advancement: Transforming Policy under Maximum Causal Entropy Principle}

\author{
  Guojun Wu \\
  WPI \\
  \And
  Yanhua Li \\
  WPI \\
  \And
  Zhenming Liu\\
  College of William \& Mary \\
  \AND
  Jie Bao\\
  JD Finance \\
  \And
  Yu Zheng\\
  JD Finance \\
  \And
  Jieping Ye\\
  Didi Chuxing \\
  \And
  Jun Luo \\
  Machine Intelligence Lab, \\
  Lenovo Group Limited \\
}

\begin{document}
\maketitle
\begin{abstract}
Many real world human behaviors can be characterized as sequential decision making processes, such as urban travelers' choices of transport modes and routes~(\cite{wu2017data}). 
Differing from choices controlled by machines, which in general follows {\em perfect rationality} to adopt the policy with highest reward, studies have revealed that human agents make sub-optimal decisions under {\em bounded rationality}~(\cite{tao2014examining}). Such behaviors can be modeled using maximum causal entropy (MCE) principle~(\cite{ziebart2010modeling}). 
%
%
%
In this paper, we define and investigate 
a general reward transformation problem (namely, {\em reward advancement}): 
Recovering additional rewards that transform the agent's policy from $\pi_o$ to a predefined target policy $\pi_t$ under MCE principle. 
We show that given an MDP and a target policy $\pi_t$, there are infinite many additional reward  functions that can achieve the desired policy transformation. 
%
Moreover, we propose 
an algorithm to further extract the additional rewards with minimum ``cost'' to implement the policy transformation.
We demonstrated the correctness and accuracy of our reward advancement solution using both synthetic data and a large-scale (6 months) passenger-level public transit data from Shenzhen, China.
\end{abstract}
\section{Introduction}
In sequential decision making problems~(\cite{ziebart2010modeling}), human agents complete tasks by evaluating the rewards received over states traversed and actions employed.
Each human agent may have her own unique reward function, which governs how much reward she may receive over states and actions~(\cite{wong2015two,zhang2006agent}).
For example, urban travelers may evaluate the travel cost vs travel time with different weights, when deciding which transport mode, route, and transfer stations to take~(\cite{wu2017data}). Uber drivers may prefer different urban regions to look for passengers, depending on their familiarity to the regions, and distance to their home locations, etc~(\cite{wu2018report}).
To quantify and measure the unique reward function each human agent possesses, maximum causal entropy inverse reinforcement learning (IRL)~(\cite{ziebart2008maximum}) has been proposed to find the reward function and the corresponding policy, that best represents demonstrated behaviors from the human agent with the highest causal entropy, subject to the constraint of matching feature expectations to the distribution of demonstrated behaviors.

Going beyond the human agent reward learning problem, in this paper, we move one step further to investigate how we can influence and change agent's policy (i.e., decisions) to a target policy $\pi_t$ from the original policy $\pi_o$ with minimum cost, by purposely updating and advancing the rewards received by the human agent.

Figure~\ref{fig:example} illustrates this problem with a concrete example in the public transit setting. We assume passengers want to travel from $Src$ to $Dst$, obviously, there are two different routes. The first one is to take $bus\#2$ to $Dst$ and the other route is take $bus\#2$ to $G1$ and take subway to $Dst$. The reward of first path is $4$ and the second one is $1$. Naturally, the probability to take $bus\#2$ is $0.8$. If we want to balance the passengers flow between two routes, we need to provide additional reward of the second route. And it will lead to a balanced passengers flow. However, there are multiple options in terms of providing additional rewards. For example, in the Figure~\ref{fig:example} c, we provide $r=3$ to taking $bus\#1$, but we can have another option like providing $r=1$ to $bus\#2$ and $r=4$ to $bus\#1$. Then, the question is, how to calculate the optimal pattern in terms of minimizing the total additional reward we provide.

This problem of finding additional reward to transform human agent's policy with minimum cost is of crucial practical importance. For example, passengers of urban transit system like buses and subways always have their own policy while travelling for example, which bus or subway line to take. However, due to the lack of knowledge about other passengers' decisions, their individual policy would usually cause unbalanced distribution of passengers both spatially and temporally. To mitigate this problem, the government can design a global optimal policy for each passenger based on global information. However, simply asking passengers to follow that policy is hardly possible, for example, if we tell someone to start her trip to work one hour before her normal schedule, we have no chance she will follow. To transform agents' policy to our specific designed policy, we need to provide additional reward to the those agents such as providing discounted price if she start earlier~(\cite{zheng2014urban,lachapelle2011commuting}). Then, how to minimize the cost of transforming agents' policy is critical. 

In the literature, reward transformations ~(\cite{wiewiora2003principled,ng1999policy,konidaris2012transfer,devlin2011theoretical}) have been studied extensively, primarily focusing on transforming the reward, with the goal of preserving the same policy (which is formally termed as ``reward shaping''). Differing from reward shaping, our design goal is more general, namely, transforming rewards, so the agent behaves as a target policy $\pi_t$, which may or may not be the agent's original policy $\pi_o$. We refer this problem as a ``reward advancement'' problem.
\begin{figure}[!t]
	\minipage{1\textwidth}
    \centering
	\includegraphics[width=1\textwidth]{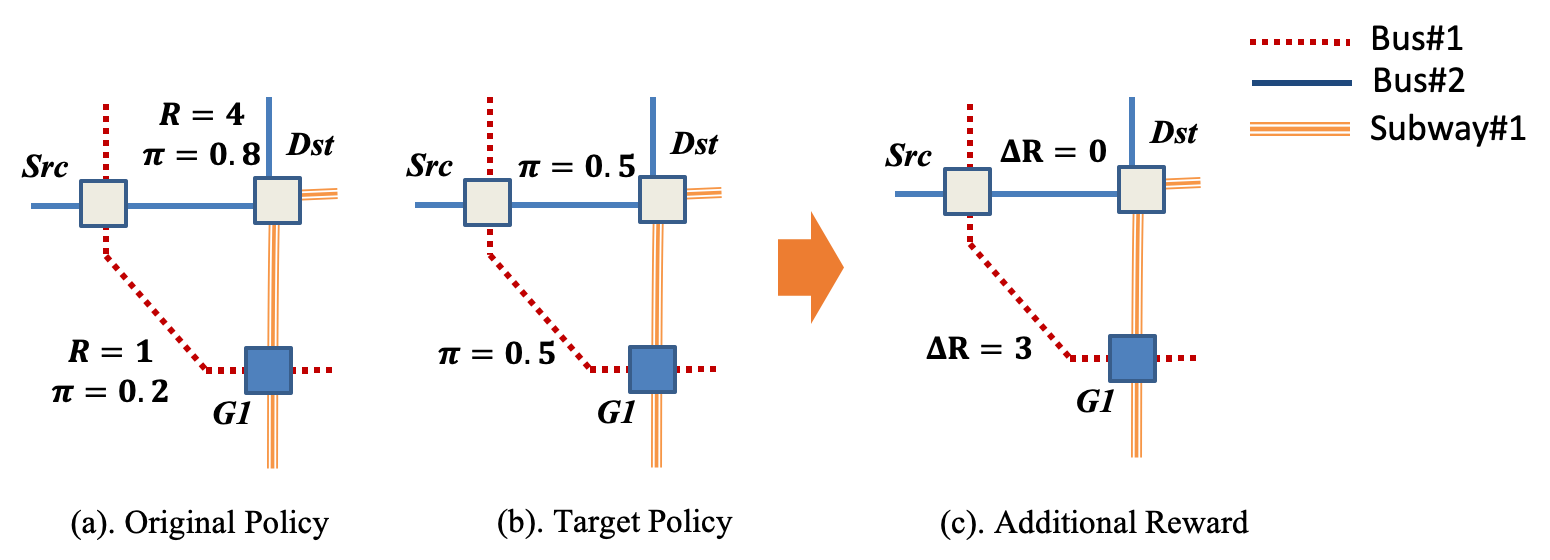}
	\vspace{-4mm}
	\caption{An Example of Reward Advancement}\label{fig:example}
	\endminipage
\end{figure}

In this paper, we make the first attempt to tackle the reward advancement problem. 
Given a Markov Decision Process and a target policy $\pi_t$, we investigate the range of additional rewards that can transform the agent's policy to the predefined target policy $\pi_t$ under MCE principle. 
Our main contributions are summarized as follows.

\begin{itemize}
	\item We are the first to define and study the reward advancement problem, namely, finding the updating rewards to transform human agent's behaving policy to a predefined target policy. We provide a close-form solution to this problem. The solution indicates that there exist infinite many such additional rewards, that can achieve the desired policy transformation.
	\item  Moreover, we define and investigate min-cost reward advancement problem, which aims to find the  additional rewards that can transform the agent's policy to $\pi_t$, while minimizing the cost of the policy transformation. 
    \item We also demonstrated the correctness and accuracy of our reward advancement algorithm using both synthetic data and a large-scale (6 months) passenger-level public transit data from Shenzhen, China.

\end{itemize}


\section{Preliminaries}
\label{sec:pre}
  


In this section, we review the basics of finite Markov Decision Process and Maximum Causal Entropy (MCE) policy. 

\subsection{Markov Decision Process (MDP)}
An MDP is represented as a tuple $\langle S, A, T, \gamma, \mu_0, R\rangle$,
where $S$ is a finite set of states and $A$ is a set of actions. 
$T$ is the probabilistic transition function with $T(s'\mid s,a)$
as the probability of arriving at state $s'$ by executing
action $a$ at state $s$, $\gamma \in (0,1]$ is the discounting factor\footnote{Without loss of generality, we assume $\gamma=1$ in this paper, where our results can be extended to the case with $\gamma \in (0,1]$.},
$\mu_0: S\rightarrow [0,1]$ is the initial distribution, and
$R:S\times A\rightarrow \reals$ is the reward function. 
A randomized, memory-less policy is a function that specifies a probability distribution on the action to be executed in each state, defined as $\pi: S \times A\rightarrow [0,1]$.
The planning problem in an MDP aims to find a policy $\pi$, such that the
expected total reward is maximized, namely, 
\begin{align}
\pi^\ast &= \arg\max_{\pi\in \Pi}\mathbb{E}^\pi(\sum_{t=0}^T \gamma^t
R(S_t,A_t) \mid S_0\sim \mu_0), \label{eq:opt_pi}
\end{align}
where $S_t$ and $A_t$ are random variables for the state and action at
the time step $t$, and $T\in \reals\cup\{\infty\}$ is the set of time horizons. 
The initial state
$S_0$ follows the initial distribution $\mu_0$. Here, $\Pi$ is the
memory-less policy space.

\subsection{Policy under Maximum Causal Entropy Principle}

Optimal policy outlined in eq.(\ref{eq:opt_pi}) achieves the highest expected reward for the agent. It is widely used for machine (i.e., robot) agent design~(\cite{ng2000algorithms,levine2014learning,li2004iterative}), where {\em perfect rationality} can be safely assumed~(\cite{li2004iterative}). 
However, many studies have revealed that decisions made by human agents (even experts) are probabilistic and sub-optimal~(\cite{kuefler2017imitating,wu2017data,tao2014examining}). These phenomena indicate that {\em human agents} are making decisions with bounded rationality~(\cite{wu2017data,tao2014examining}), where actions are chosen with probabilities corresponding to the expected future rewards they are leading to. %
As a result, various inverse reinforcement learning algorithms were proposed to recover the reward function, $R: S\times A\rightarrow \mathbb{R}$ such that the distribution of action and state sequence under a near-optimal policy match the demonstrated human behaviors. 

One well-known solution to the inverse reinforcement learning problem is Maximum Causal Entropy Inverse Reinforcement Learning~(\cite{ziebart2010modeling}). It proposes to find the policy that best represents demonstrated behaviors with highest causal entropy, which is summarized as follows.

\noindent{\bf Conditional Entropy}
 is used to measure the uncertainty of one distribution $A$ based on a given side information $S$, i.e., $H({A}|{S})=-\sum_{a\in A}\sum_{s\in S}P(s,a)\ln P(a|s)$. 

\noindent{\bf Causal Entropy} measures the uncertainty present in the causally conditioned
distribution of a sequence variable $A_{0:T}$, given the preceding partial sequences $S_{0:t}$ and $A_{0:t-1}$, with $0\leq t\leq T$. It
can be interpreted as the expected number of bits needed to encode the sequence $A_{0:t}$ given the
previous $A_{0:t-1}$ variables and sequentially revealed side information, $S_{0:t}$, which has been revealed at
each point in time and excluding unrevealed future side information, $S_{t+1:T}$, which is $H(A||S)=\sum_{t=0}^{T}H(A_t|A_{0:t-1},S_{0:t}).$
When the sequence is Markovian, the causal entropy can be written as $H(A||S)=\sum_{t=0}^{T}H(A_t|S_{t})$. As a result, the causal entropy of an MDP is characterized as $H(A||S)=-\sum_{s\in S}\sum_{a\in A}D(s,a)\ln\pi(a|s)$ with $A$ and $S$ as the action sequence and the state sequence (side information), respectively. And $D(s,a)$ represents the expected visitation frequency of the state-action pair $(s,a)$, when one trajectory is generated under policy $\pi(a|s)$.

The policy under maximum causal entropy principle (i.e., {\em MCE policy}) best represents the demonstrated behaviors with the highest causal entropy, and is subject to the constraint of matching reward expectations to the distribution of demonstrated behaviors. Denote $Q(s,a)=R(s,a)+\sum_{s^\prime\in S}T(s^\prime|s,a)\sum_{a^\prime\in A}\pi(a^\prime|s^\prime)Q(s^\prime,a^\prime)$ as Q-function on state-action pair $(s,a)$, indicating the expected rewards to be received starting from $(s,a)$, MCE policy can be formulated as the following maximum causal entropy problem: 

\begin{align}
\textbf{Problem 1:} &\textbf{ Maximum Causal Entropy Policy:}\nonumber \\
\max_{\pi(a|s)}\quad\quad & H(A||S)=-\sum_{s\in S}\sum_{a\in A}D(s,a)\ln\pi(a|s),\label{eq:MCE_obj} \\ 
s.t. \quad\quad&\sum_{s\in S}\sum_{a\in A}\mu_0(s)\pi(a|s)Q(s,a)=\tilde{Q}, \label{eq:MCE_c1} \\
&\sum_{a\in A}\pi(a|s)=1,\forall s\in S.\label{eq:MCE_c2}
\end{align}
where $\tilde{Q}=\sum_{\zeta\in\tilde{TR}}R(\zeta)/|\tilde{TR}|$ is expected empirical rewards extracted from the behavior data. $\tilde{TR}=\{\zeta\}$ is a set of demonstrated trajectories from a human agent, and $|\tilde{TR}|$ denotes the size of the trajectory set. $R(\zeta)=\sum_{(s,a)\in \zeta}R(s,a)$ is the reward received on trajectory $\zeta$.


\begin{theoremm}\label{thm:theorem_Q_softmax}
The MCE policy characterized in {Problem 1} eq.(\ref{eq:MCE_obj})--(\ref{eq:MCE_c2}) follows the softmax format, 
$\pi(a|s)=\frac{e^{Q(s,a)}}{\sum_{a^\prime \in A}e^{Q(s,a^\prime)}}$.
\end{theoremm}
\begin{proof}	
This can be proven by introducing Lagrangian multipliers for constraints, and letting the derivative of Lagrangian function be zero. See more details in the supplementary material.
\end{proof}

%
%


\section{Reward Advancement}\label{sec:mod}

Inverse reinforcement learning problem~(\cite{ng1999policy,ziebart2010modeling,ziebart2008maximum,finn2016guided}) aims to inversely learn agent's reward function from their demonstrated trajectories, namely, inferring how agent makes decisions.
In this work, we move one step further to investigate how we can influence and change agent's policy (i.e., decision making) to a target policy $\pi_t$ from the original policy $\pi_o$ observed from the demonstrated trajectories, by purposely updating and advancing rewards ${R}({s,a})$ with $\Delta R(s,a)$ in the MDP. 
Reward transformations~(\cite{ng1999policy, wiewiora2003potential}) have been studied in the literature, primarily focusing on transforming the rewards, with the goal of preserving the same policy (which is formally termed as ``reward shaping'').
Differing from reward shaping, our design goal is more general, say, transforming rewards, so the agent behave as a predefined target policy $\pi_t$, which may or may not be the agent's current policy $\pi_o$.
This problem is referred to as a ``reward advancement'' problem, and we formally define it as follows.


\noindent{\bf Reward Advancement Problem.} Given an MDP $\langle S, A, T, \mu_0, R_o\rangle$, the agent's MCE policy is $\pi_o$. we aim to find additional rewards $\Delta R$ to be added to the original reward $R_o$, such that the agent's MCE policy under the updated MDP $\langle S, A, T, \mu_0, R_o+\Delta R\rangle$ follows a predefined target policy $\pi_t$.

For MDP $\langle S, A, T, \mu_0, R_o\rangle$, the Q-function of executing a policy $\pi$ can be expressed as $Q^{\pi}_o(s,a)=R_o(s,a)+\sum_{s^\prime\in S}T(s^\prime|s,a)\sum_{a^\prime\in A}\pi(s^\prime,a^\prime)Q^{\pi}_o(s^\prime,a^\prime)$. Then, the Q-function with additional reward $\Delta R(s,a)$ is $Q^{\pi}(s,a)=R(s,a)+\sum_{s^\prime\in S}T(s^\prime|s,a)\sum_{a^\prime \in A}\pi(a^\prime|s^\prime)Q^{\pi}(s^\prime,a^\prime)$, where $R(s,a)=R_o(s,a)+\Delta R(s,a)$, $Q^{\pi}(s,a)=Q^{\pi}_o(s,a)+\Delta Q(s,a)$, and $\Delta Q(s,a)=\Delta R(s,a)+\sum_{s^\prime\in S}T(s^\prime|s,a)\sum_{a^\prime\in A}\pi(a^\prime|s^\prime)\Delta Q(s^\prime,a^\prime)$. 


As a result, transforming from the original MCE policy $\pi_o$, the new MCE policy $\pi$ is a function of addition reward $\Delta R$, or equivalently $\Delta Q$, i.e., $\pi(a|s;\Delta Q)$. Given a predefined $\pi_t$, finding the right $\Delta Q$, such that $\pi(a|s;\Delta Q)=\pi_t(a|s)$ for any $s\in S$ and $a\in A$, solves the reward advancement problem. The following Theorem~\ref{thm:theorem_unlimited} introduces the complete solution set to this problem.
\begin{theoremm}\label{thm:theorem_unlimited}
Given an MDP $\langle S, A, T, \mu_0, R_o\rangle$, the sufficient and necessary condition to transform its MCE policy to a predefined policy $\pi_t$ is to provide additional Q-function $\Delta Q$, such that 
\begin{align}
    \Delta Q(s,a)=\ln \frac{\pi_t(a|s)}{e^{Q_o^{\pi_t}(s,a)}}+\beta(s),\label{eq:advancement}
\end{align}
where $\beta: S\rightarrow \mathbb{R}$ is any real number function defined on states. Such additional Q-function is called ``advancement function''.
\end{theoremm}
\begin{proof}
(sketch) If we set $\beta(s)=\ln\sum_{a^\prime\in A}e^{Q_o^{\pi_t}(s,a^\prime)+\Delta Q(s,a^\prime)}$, then $\beta(s)$ can be viewed as softmax sum of Q function of all state-action pair $(s,a^\prime)$ associated with $s$. Then, $\Delta Q(s,a)$ can be calculated based on the MCE based policy. See more details in the supplementary material.
\end{proof}

Theorem~\ref{thm:theorem_unlimited} indicates that there are infinite many advancement strategies, i.e. $\beta(s)$ that can transform an original MCE policy $\pi_o$ to a given $\pi_t$. However, different advancement strategies may lead to different costs in reality while the implementation of additional rewards. For example, in ride-hailing service, additional rewards provided to Uber drivers could be in the form of monetary values; in urban public transportation systems, the additional rewards to passengers could be in the form of ride discount. More additional rewards applied lead to more cost to the implement. Besides, without any lower bound on $\beta(s)$, the advancement function $\Delta Q$ can be as low as $-\infty$. In turn, the addition rewards $\Delta R$ inferred via Bellman equation can be arbitrarily small as well. It is equivalent to increase the ride rate to be extremely large for public transits, which is not feasible in real world scenario. Si, we will introduce and provide solution to the reward advancement problem with minimum cost as the objective in the following section.


\section{Min-Cost Reward Advancement}\label{sec:min-cost}

Now, we investigate how to identify additional rewards that transform the agent to an MCE policy $\pi_t$, while guaranteeing minimum ``implementation cost'', namely, a {\em min-cost reward advancement problem}.
For many real-world cases, however, we can only manipulate rewards by providing additional features like changing passenger's inherent reward by providing monetary incentives. So we take the approach that advance agent's reward by providing additional feature. For simplicity, we consider that reward function is in a linear fashion, i.e., $R(s,a) = \omega^T F(s,a)$, where $F(s,a)$ is the feature vector. Then, the additional reward can be defined as $\Delta R(s,a) = \omega^T \Delta F(s,a)$, where $\Delta F(s,a)$ is the additional feature we provide to advance agent's reward. Then, we can define the "implementation cost" as cost of providing additional features, which is given by $C(s,a) = \phi^T \Delta F(s,a)$, where $\phi$ is the cost function.

Before we jump to details, we can make the assumption that $\frac{\omega_i}{\phi_i}>0$, where $\omega_i$ and $\phi_i$ is the i-th entry in $\omega$ and $\phi$, since that in real-world application, if there is something we can provide to make both side (like passengers and drivers) happy, which means reduce cost and increase reward of passengers, we ought to provide this feature as much as possible. Based on this assumption, we can then take a look at how to assign additional reward to different features efficiently. The constraints on each feature can be denoted as,
\begin{equation}
    \begin{cases}\label{con:raw_feature}
        &\phi_i \Delta F_i(s,a)\geq c_i^{min}, \forall i, \\
        &\phi_i \Delta F_i(s,a)\leq c_i^{max}, \forall i,
    \end{cases}
\end{equation}
where $c_i^{min}$ and $c_i^{max}$ can be any real value. Of course if you set $c_i^{min}>c_i^{max}$ then there will exist no valid solutions. 

Based on the constraints we have on cost, we can define whether a given additional reward $\Delta \tilde{R}(s,a)$ is achievable,
\begin{dfr}
    For a given $\Delta \tilde{R}(s,a)$, if $\sum_i\frac{\omega_i}{\phi_i}c_i^{min}\leq\Delta \tilde{R}(s,a)\leq \sum_i\frac{\omega_i}{\phi_i}c_i^{max},$ we call the additional reward $\Delta \tilde{R}(s,a)$ achievable. 
\end{dfr}
It is obvious that $\sum_i\frac{\omega_i}{\phi_i}c_i^{min}$ and $\sum_i\frac{\omega_i}{\phi_i}c_i^{max}$ is the lower and upper bound of additional reward that can be provide by altering features without violating any constraints, i.e.  if a target policy $\tilde{\pi_t}(a|s)$ needs additional reward exceed those bounds are not achievable. For example, we can't convert all private car owners to bus-takers, which is theoretically probable if we award every one, say, \$10,000 per bus trip, but it far more exceeds governments budget.  

\begin{theoremm}\label{thm:assignment}
    Given additional rewards $\Delta \tilde{R}(s,a)$, if $\Delta \tilde{R}(s,a)$ is achievable, then we have the minimum cost of $\Delta \tilde{R}(s,a)$, \[C^{min}(\Delta \tilde{R}(s,a)) = \sum_{i=0}^{i<N}c_i^{min}+\sum_{j=0}^{j<N}\max[\min(\Delta \tilde{R}(s,a)-r^{min}-\sum_{k=0}^{k<j}r_k^{max}, r_j^{max}), 0]\frac{\phi_j}{\omega_j},\]
    where $r^{min}=\sum_{i=0}^{i<N}\frac{\omega_i}{\phi_i}c_i^{min}$ is the lower bound of achievable additional reward,  $r_i^{max}=\frac{\omega_i}{\phi_i }c_i^{max}$ is the upper bound of additional reward at each feature, $j$ is the index of feature list sorted by descending order of $\frac{\omega_i}{\phi_i}$
\end{theoremm}
\begin{proof}
(sketch) We first sort the feature by their cost-efficiency, which is $\frac{\omega_i}{\phi_i}$. Then, we start from the minimum achievable value of $\Delta R(s,a)$, which is $\sum_{i=0}^{i<N}c_i^{min}$ and pick features to provide additional reward according to their cost-efficiency, then we would get the most cost-efficient way of assignment by this greedy method. 
\end{proof}

So far, we have successfully find a way to assign a give additional reward $\Delta R(s,a)$ to its corresponding features $\Delta F(s,a)$. Then, we can start to answer the question: what is the additional reward with minimum cost? 

\begin{preposition}\label{pos:same_order}
    If $\Delta R_1(s,a)\leq\Delta R_2(s,a)$, then $C(\Delta R_1(s,a)) \leq C(\Delta R_2(s,a))$.
\end{preposition}
We use Preposition~\ref{pos:same_order} to show that the min-cost reward advancement is actually a min-reward advancement problem with upper and lower bound constraints. Then, constraints on additional reward, which is $r^{min}\leq\Delta R(s,a)\leq r^{max}$, where $r^{max}=\sum_{i=0}^{i<N}\frac{\omega_i}{\phi_i}c_i^{max}$ and $r^{min}=\sum_{i=0}^{i<N}\frac{\omega_i}{\phi_i}c_i^{min}$ are the upper and lower bound of additional reward based on constraints on $\Delta F(s,a)$. So, the min-cost reward advancement problem can be formulated as a two-stage problem, first learning the minimum additional rewards, $\Delta R^*(s,a)$, which we should provide to transform agent's policy to $\pi_t$, defined as \textbf{ Min-Reward Stage}. And then assign it to features based on eq.~\ref{eq:assignment}, also known as \textbf{Assignment Stage}.  

\textbf{Min-Reward Stage.}  The min-reward reward advancement can be formulated as,
\begin{align}
\textbf{Pro}&\textbf{blem 2: Min-Reward Reward Advancement}\nonumber\\
    \min \quad\quad& \sum_{s\in S} \mu_0(s)\sum_{a\in A}\pi_t(a|s)\Delta Q(s,a) \label{obj:fea} \\
    s.t. \quad\quad& \Delta Q(s,a)-\sum_{s^\prime\in S}\sum_{a^\prime\in A}T(s^\prime|s,a)\pi_t(a^\prime)|s^\prime\Delta Q(s^\prime,a^\prime)=\Delta R(s,a), \forall s\in S,a\in A \label{cst:q_and_r} \\
    &\Delta Q(s,a)=\ln\pi_t(a|s)-Q_o^{\pi_t}(a|s)+\beta(s), \forall s\in S,a\in A \label{cst:pi_t} \\
    & r^{min}\leq\Delta R(s,a)\leq r^{max}, \forall s\in S,a\in A \label{cst:r_bound} 
\end{align}
where $\Delta Q(s,a)$ is the additional Q-function and $\Delta R(s,a)$ is the additional reward. 
The objective function eq.~\ref{obj:fea} is to get the minimum additional reward we should provide, which is equivalent to minimum cost of providing those reward according to Preposition~\ref{pos:same_order}.
Constraint eq.~\ref{cst:q_and_r} reflects the relationship between additional Q-function $\Delta Q(s,a)$ and the additional reward $\Delta R(s,a)$. And eq.~\ref{cst:pi_t} is used to guarantee the target policy $\pi_t$ can be achieved after providing additional reward $\Delta R(s,a)$. Since we are using MCE policy assumption throughout this paper, here we adopt the result of Theorem~\ref{thm:theorem_unlimited}. One may use other randomized policy assumption and the problem still would have similar solution. Besides, we have additional reward constraints eq.~\ref{cst:r_bound}, representing the upper and lower bound in terms of additional reward we can provide to the agent. 

\begin{theoremm}\label{thm:min_reward_problem_solution}
    If we define $k(s,a)=\log\pi_t(a|s)-\sum_{s^\prime\in S}\sum_{a^\prime\in A}T(s^\prime|s,a)\pi_t(a^\prime|s^\prime)\log\pi_t(a^\prime|s^\prime)-R_o(s,a), \forall s\in S a\in A$, then the solution to Problem 2, the Min-Reward Reward Advancement Problem, can be written as,
    \begin{equation}
        \begin{cases}
            &\Delta R^*(s,a)=\beta_{min}(s)-\sum_{s^\prime\in S}T(s^\prime|s,a)\beta_{min}(s^\prime)+k(s,a), \beta_{min}(s)\leq\beta_{max}(s), \forall s\in S \\
            &\text{No Valid Solution}, \beta_{min}(s)\leq\beta_{max}(s), \exists s\in S, a\in A,
        \end{cases}
    \end{equation}
    where $\beta_{min}(s)$ and $\beta_{max}(s)$ is the lower bound and upper bound of $\beta(s)$ in Problem 2 and is given by,
    \begin{equation}\label{eq:bellman}
        \begin{cases}
            \beta_{min}(s)= \max_{a\in A}\big(\sum_{s^\prime\in S}T(s^\prime|s,a)\beta(s^\prime)+ r^{min}-k(s,a)\big) \\
            \beta_{max}(s)=\min_{a\in A}\big(\sum_{s^\prime\in S}T(s^\prime|s,a)\beta(s^\prime)+ r^{max}-k(s,a)\big)
        \end{cases}
    \end{equation}
\end{theoremm}
\begin{proof}
(sketch) The $\beta(s)$ can be viewed as value function of state $s$ with an arbitrary reward function. If is easy to understand that when we set constraints on additional reward of each state-action pair, we actually are setting constraints on the value function of each state and we can use value iteration to solve this problem. The full proof can be found in supplementary material. 
\end{proof}

\textbf{Assignment Stage.} While we extract $\Delta R^*(s,a)$ for each state-action pair $(s,a)$, we still need to assign additional rewards to different features to assure the minimum cost of transforming policy. 
The Theorem~\ref{thm:assignment} indicates that for each possible $\Delta R(s,a)$, there exists a assignment of additional features to achieve minimum transfer cost, which is,
\begin{equation}\label{eq:assignment}
    \Delta F_i(s,a)=\max[\min(\Delta \tilde{R}(s,a)-r^{min}-\sum_{k=0}^{k<i}r_k^{max}, r_i^{max}), 0]\frac{1}{\omega_i}
\end{equation}

\begin{algorithm}[t]
	\caption{Min-Cost Reward Advancement via Value Iteration} ~\label{alg:tpra_cost}
	\begin{algorithmic}[1]
		\State {\textbf{INPUT:}} States $S$, Actions $A$, Original Rewards $R_o$ and Original Trajectory Set $\tilde{TR}$ and cost constraint $c^{min}_i$ and $C^{max}_i$;
		\State {\textbf{OUTPUT:}} Additional reward on each state-action pair $\Delta R^*(s,a)$ (One from many solutions);
		\State Calculate $r^{max}=\sum_{i=0}^{i<N}\frac{\omega_i}{\phi_i}c_i^{max};$
		\State Calculate $r^{min}=\sum_{i=0}^{i<N}\frac{\omega_i}{\phi_i}c_i^{min};$
        \State For each state-action pair $(s,a)$, calculate $k(s,a)=\log\pi_t(a|s)-\sum_{s^\prime\in S}\sum_{a^\prime\in A}T(s^\prime|s,a)\pi_t(a^\prime|s^\prime)\log\pi_t(a^\prime|s^\prime)-R_o(s,a)$;
        \State For each state-action pair $(s,a)$, calculate $R^\prime_{max}(s,a)=r^{max}-k(s,a)$ and $R^\prime_{min}(s,a)=r^{min}-k(s,a);$
        \State Use $R^\prime_{max}(s,a)$ and $R^\prime_{min}(s,a)$ as rewards to perform value iteration to calculate the lower bound $\beta_{min}(s)$ and upper bound $\beta_{max}(s)$;
        \For {Each state-action pair $(s,a)$}
            \If {$\beta_{min}(s)>\beta_{max}(s)$}
                \State Return NO VALID SOLUTION
            \EndIf
            \State Calculate $\Delta R^*(s,a)=\beta_{min}(s)-\sum_{s^\prime\in S}T(s^\prime|s,a)\beta_{min}(s^\prime)+k(s,a)$
        \EndFor
		\State For each state-action pair $(s,a)$, calculate $\Delta F_i(s,a)=\max[\min(\Delta R^*(s,a)-r^{min}-\sum_{k=0}^{k<i}r_k^{max}, r_i^{max}), 0]\frac{1}{\omega_i}$
	\end{algorithmic}
\end{algorithm}

The Algorithm.~\ref{alg:tpra_cost} demonstrates how to get the optimal solution to the min-cost reward advancement problem. First, we use Line 3 and Line 4 to calculate bounds of additional rewards. And Line 5~Line 11 is used to calculate the minimum additional reward we can provide to transform the agents policy. For some cases that $T(s^\prime|s,a)$, which is the transition probability, is missing, we can use samples from trajectories $\tilde{TR}$ to estimate the expectation in Line 5, Line 7 and Line 11. Then, finally, the Line 12 produces optimal additional feature assignments to each state-action pair. 
\section{Evaluation}\label{sec:eva}

In this section, we evaluate the correctness and accuracy of our min-cost reward advancement algorithm, with synthetic object world scenario. Then, by modeling passengers' travel decisions in public transit system as a Markov Decision Process, we conduct empirical case studies using a large-scale (6 months) passenger-level public transit data collected in Shenzhen, China, from 07/01/2016 to 12/30/2016.

\textbf{Evaluation on object world.} First, we use an object world~(\cite{levine2011nonlinear}) scenario to evaluate our reward advancement algorithm. A Object World is a Grid World with random placed colored object and running into grid with object with different color will have different reward, we call it "collect the object". And the agent will also get a large reward by arriving the destination. So, the ideal policy should be going to the destination while collect object with higher reward as many as possible. Figure~\ref{fig:object_world} shows an example of object world. There are $5 \times 9$ grids. We randomly placed $2$ green objects and $3$ red objects in the scenario. At each grid, agent can take $5$ different actions, including stay and move towards one of four directions. With certain given transition probability, the agent would go to a random neighboring grid along the direction she has chosen. We set the discount factor to be $1$ for all experiments. Then, we use Figure~\ref{fig:tpra_traj} to show the efficiency and effectiveness of the min-cost reward advancement algorithm. For object world have more that $500$ state-action pairs, only $50$ trajectories are needed to learn an accurate additional reward to transform agent's policy to a predefined policy. The Figure~\ref{fig:tpra_cost} shows that the total cost of reward advancement increase linearly while the lower bound of additional reward at each $(s,a)$ increases, which demonstrates that the cost of transforming policy via min-cost reward advancement is applicable.

 \begin{figure*}[!t]
    \minipage{0.3\textwidth}
	\includegraphics[width=1\textwidth]{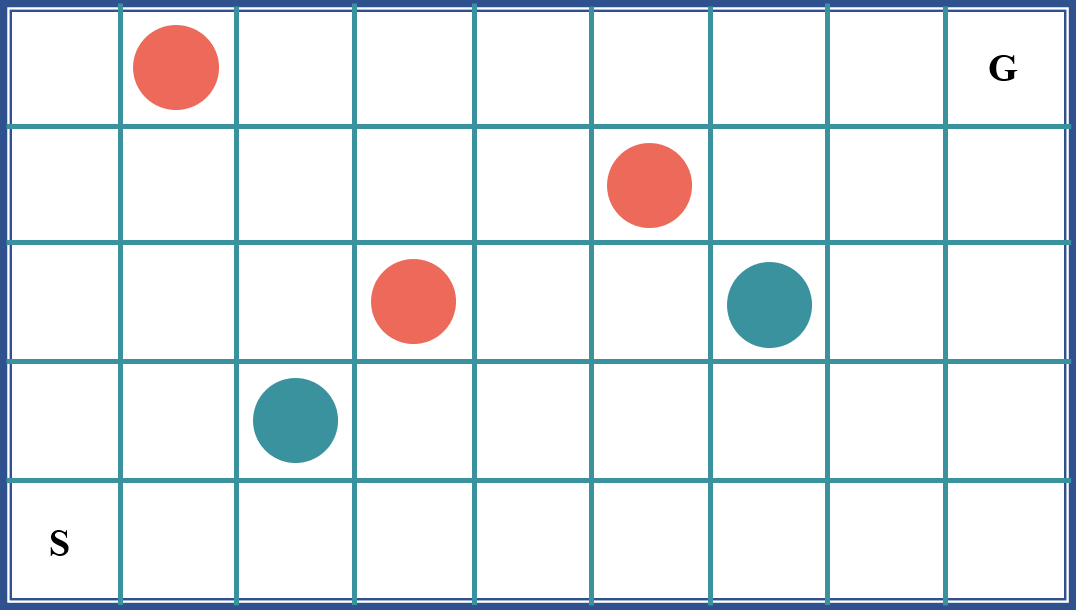}
	\vspace{0mm}
    \caption{A $5\times9$ Object World with 2 different colors.}\label{fig:object_world}
	\endminipage
	\hspace{3mm}
	\minipage{0.3\textwidth}
	\includegraphics[width=1\textwidth]{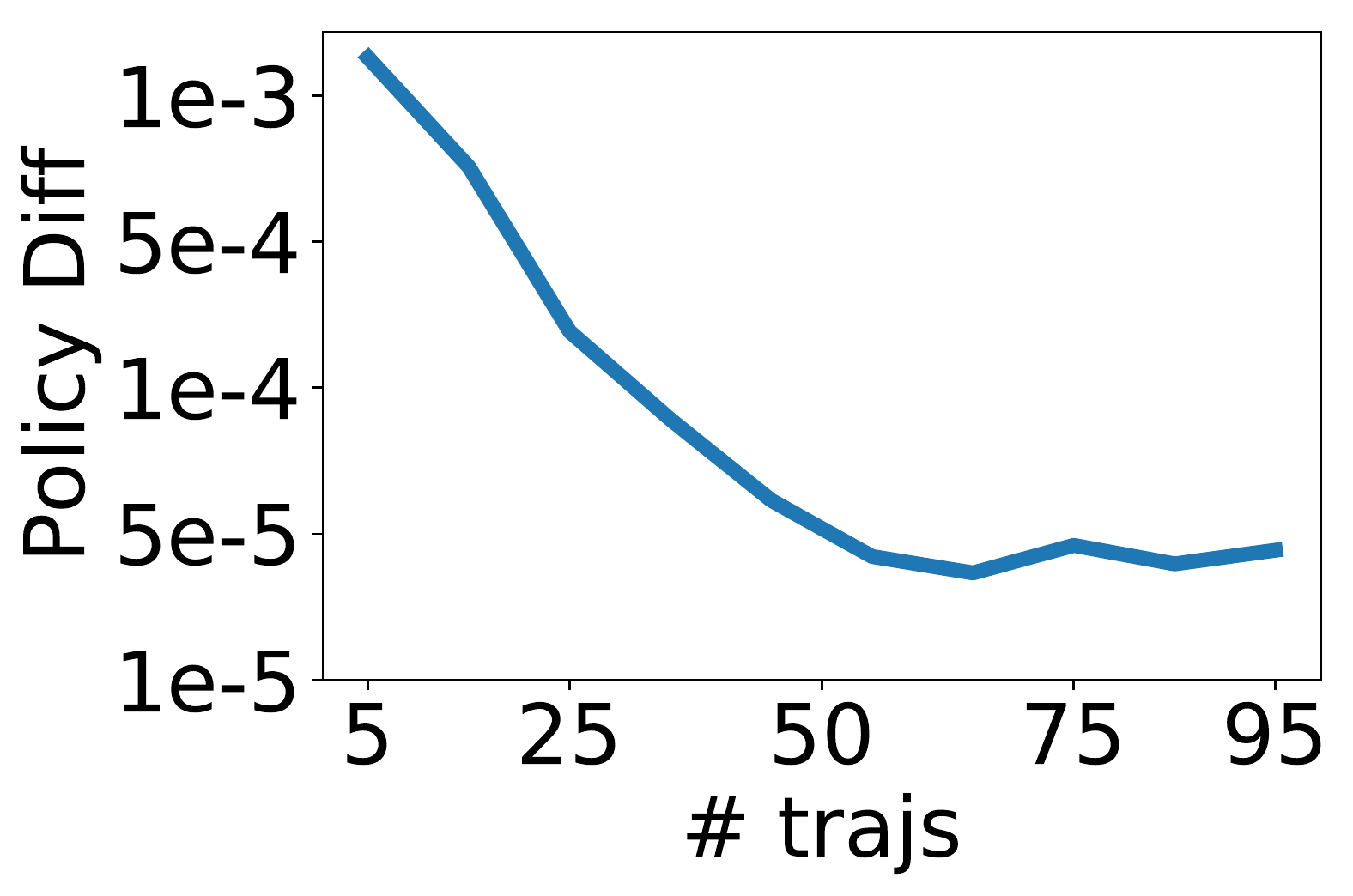}
	\vspace{-3mm}
	\caption{Number of trajectories used vs accuracy.}\label{fig:tpra_traj}
	\endminipage
	\hspace{3mm}
	\minipage{0.3\textwidth}
	\includegraphics[width=1\textwidth]{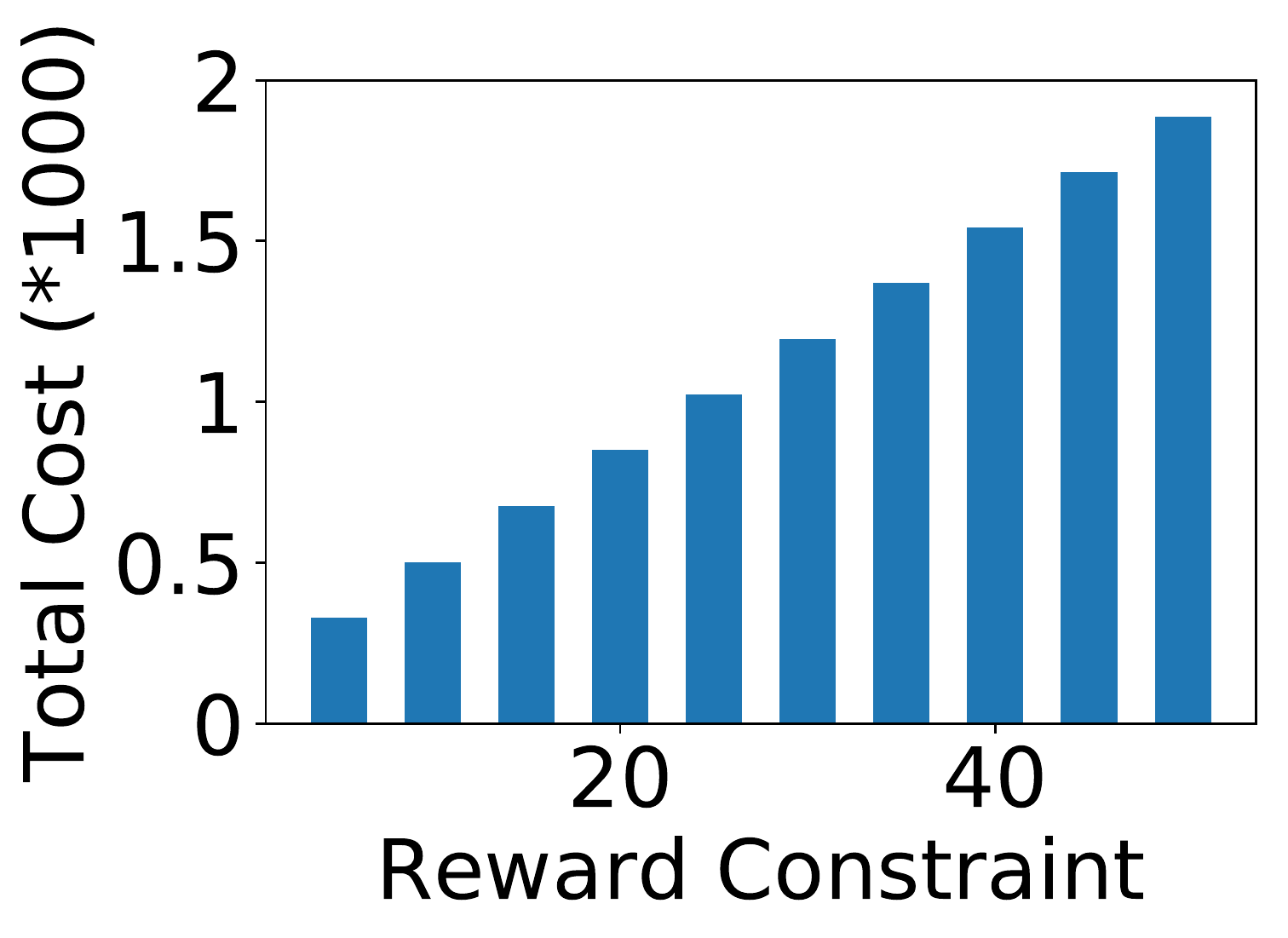}
	\vspace{-6mm}
	\caption{Total cost over reward lower bound.}\label{fig:tpra_cost}
	\endminipage
\end{figure*}
\textbf{Case studies.}  In this section, we will use a public transit case as an example to illustrate that human agents' behaviors follow MCE policy and reward advancement strategy. We collected 6 months passenger-level public transit data from Shenzhen, China, which allows us to evaluate the potential of redistributing passengers by transforming their decision policies, in trip starting time, station and transport mode selection. 


Passengers are making a sequence of decisions when completing a trip, such as which bus routes and subway line to take, which stop/station to transfer. Such sequential decision making processes can be naturally modeled as Markov decision processes (MDPs). And since nearby stops/stations usually are similar to passengers, we will split the whole city into grid and aggregate stops/stations within same grid together. The states are regional grids during different time intervals. Actions are available bus routes and subway lines the passenger can take. Our model and formulation follow the work~(\cite{wu2018report}) (See ~(\cite{wu2018report} )for more details). Also, the evaluation in~(\cite{wu2018report}) indicates that the human agent would follow a MCE policy after changing of reward, which means providing additional reward can shape human agent's behaviors to the target policy.
We inversely learn the reward functions of passengers using Maximum Causal Entropy Inverse Reinforcement Learning~(\cite{ziebart2010modeling}). We consider the reward passengers are evaluating contains the monetary cost of the trip. Thus, we model the additional rewards as the monetary incentive for taking bus and subway lines.
\begin{figure*}[!t]
    \centering
    \includegraphics[width=1\textwidth]{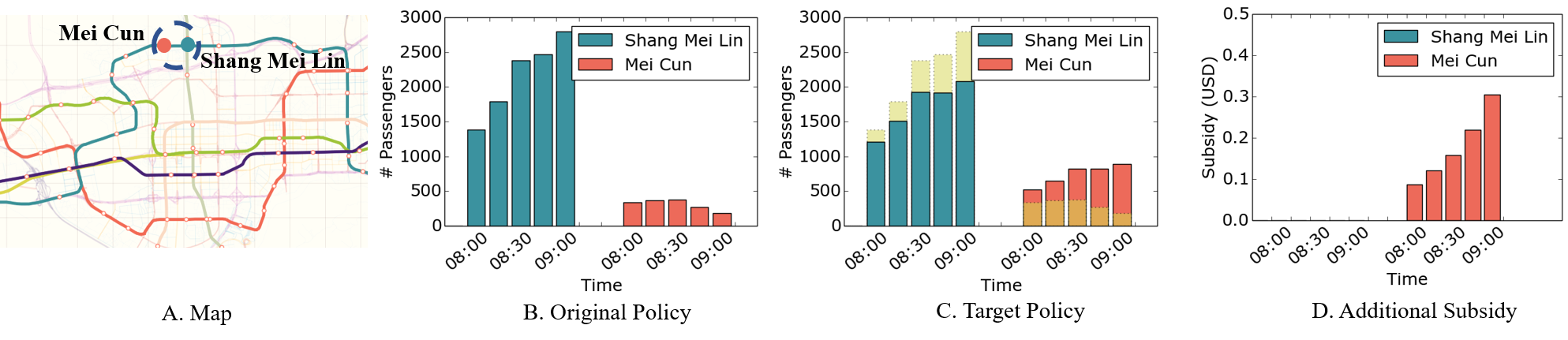}
    \vspace{-5mm}
    \caption{Spatial Decision Policy Transformation}
    \label{fig:case2}
    \vspace{-5mm}
\end{figure*}
There always exists spatial decision imbalance in public transit system. For example, Figure~\ref{fig:case2}(a) shows that there are two subway stations, Mei Cun and Shang Mei Lin, which are geographically close to each other. However, from Figure~\ref{fig:case2}(b), there are much more passengers traveling via Shang Mei Lin station rather than Mei Cun station. 
One target policy (as shown in Figure~\ref{fig:case2}(c)) we used in the experiment allow $75\%$ passengers going through Shang Mei Lin station at each $15$-minute time span. The result of reward advancement is showed in Figure~\ref{fig:case2}(d). Clearly, the additional rewards needed to transform their policy varies over time, which suggests a dynamic pricing mechanism to advance the passengers spatial decision policy.

\section{Conclusion}
In this work, we define and study a novel reward advancement problem, namely, finding the updating rewards to transform human agent's behavior to a predefined target policy $\pi_t$. We provide a close-form solution to this problem. The solution we found indicates that there exist infinite many such additional rewards, that can achieve the desired policy transformation. Moreover, we define and investigate min-cost reward advancement problem, which aims to find the additional rewards that can transform the agent's policy to $\pi_t$, while minimizing the cost of the policy transformation. We solve this problem by developing an efficient algorithm. We demonstrated the correctness and accuracy of our reward advancement solution using both synthetic data and a large-scale (6 months) passenger-level public transit data from Shenzhen, China.

\small
\bibliographystyle{unsrtnat}  
\bibliography{references}


\end{document}